\documentclass{article}
\usepackage[utf8]{inputenc}
\usepackage[T1]{fontenc}
\usepackage{authblk}
\usepackage{setspace}
\usepackage[margin=1in]{geometry}
\usepackage{graphicx}
\graphicspath{ {figs/} }
\usepackage{amsmath,amssymb,amsthm}
\usepackage{bm}
\usepackage{booktabs}
\usepackage[section]{placeins}
\usepackage{array}
\usepackage{makecell}
\usepackage[labelfont=bf,labelsep=period]{caption}
\usepackage{xcolor}
\usepackage{tikz}
\usetikzlibrary{arrows.meta,positioning,calc,shapes.geometric}
\usepackage[colorlinks=true,linkcolor=blue!50!black,citecolor=blue!50!black,urlcolor=blue!50!black]{hyperref}

\usepackage[numbers,square,sort&compress]{natbib}

\newtheorem{proposition}{Proposition}
\DeclareMathAlphabet\mathbfcal{OMS}{cmsy}{b}{n}
\newcommand{\ten}[1]{\mathbfcal{#1}}
\newcommand{\mat}[1]{\mathbf{#1}}
\newcommand{\Op}{\mathbf{Op}_m}

\title{Low-Rank Kolmogorov--Arnold Generator Flows for Interpretable Operator Learning}

\author[1*]{Minjong Cheon}
\affil[1]{Department of Computer Science, Korea National Open University, Seoul, South Korea.}
\affil[*]{Address correspondence to: jmj2316@mail.knou.ac.kr}

\date{}

\begin{document}

\maketitle

\begin{abstract}
Neural operators that learn physical laws from data should generalize beyond their training distribution and expose the learned law in closed form. The Kolmogorov--Arnold neural operator achieves both for position-dependent dynamics by learning a dense pseudo-differential symbol over all positions and frequencies, at a cost that grows with the square of the resolution and with closed forms that must be read out by hand. This dense symbol is unnecessary: a finite-order variable-coefficient differential operator has a symbol whose separation rank is bounded by the number of its coefficient functions, independently of the resolution, and the symbols that this network learns on its own benchmarks have rank two. We introduce the Kolmogorov--Arnold flow (KAFlow), which learns the generator of a physical operator as a low-rank symbol built from univariate Kolmogorov--Arnold functions, fixes the gauge of these functions so that every term has a unique place, integrates the dynamics with exact exponential integrators, and distills its closed form automatically. On the synthetic benchmarks of the Kolmogorov--Arnold neural operator, KAFlow lowers the out-of-distribution loss by 3.4 to 11 times and recovers every coefficient to within $10^{-4}$ without symbolic templates, reaching the precision floor of the data. On long-horizon quantum dynamics, it lowers the state infidelity by up to 6 orders of magnitude, and by 30 times when trained only on position and momentum measurements. KAFlow also learns a variable-order fractional operator 15 times more accurately than existing operator networks and discovers nonlocal, fractional and non-polynomial dispersion relations on which sparse regression fails.
\end{abstract}

\section{INTRODUCTION}
Data-driven modeling of physical dynamics asks for two properties at once: a learned model should generalize beyond its training distribution, and it should be interpretable, so that the learned law can be extracted, verified and reused~\cite{SciAI1,SciAI2}. Physical dynamics are formalized as operators through partial differential equations (PDEs)~\cite{pdephysics1}, and an important subclass consists of variable-coefficient PDEs~\cite{vcpde}, whose coefficients depend on position; an example is the Schr\"odinger equation with a position-dependent potential~\cite{sakurai}. Operator networks such as the Fourier neural operator (FNO)~\cite{fno,no} and DeepONet~\cite{deeponet} approximate such operators from data, but their learned representations are not closed-form laws. The Kolmogorov--Arnold neural operator (KANO)~\cite{kano} showed that position-dependent dynamics can be learned with robust out-of-distribution generalization and symbolic interpretability by parameterizing a pseudo-differential symbol $\boldsymbol p(\mat x,\bm\xi)$ jointly in position and frequency with Kolmogorov--Arnold networks (KANs)~\cite{kan,kan2}.

KANO evaluates its symbol on the full product grid of positions and frequencies and applies it through the Kohn--Nirenberg quantization~\cite{kohnNirenberg}, a double sum over positions and frequencies. The symbol is therefore represented as a dense $m\times m$ object, although the operators it targets are not dense in this sense: a variable-coefficient differential operator $\sum_k c_k(x)\partial_x^k$ has the symbol $\sum_k c_k(x)(i\xi)^k$, a sum of a few products of a function of $x$ and a function of $\xi$. Moreover, KANO learns a static map. Physical laws are generators of evolution, and the low-rank structure of a generator is lost in the propagator that a static map must represent. Finally, the closed form of a learned KANO is read out by a human who inspects the KAN edges and fixes symbolic templates for retraining~\cite{kano,kanguide}.

In parallel, equation discovery recovers governing PDEs by sparse regression over a library of candidate terms~\cite{sindy,pdefind,weaksindy}, and PDE-Net~\cite{pdenet,pdenet2} learns local differential stencils together with a symbolic network. These methods assume local, polynomial libraries, whereas many physical systems are governed by nonlocal or fractional operators: the Hilbert transform in Benjamin--Ono internal waves~\cite{benjamin,ono}, the fractional Laplacian in anomalous diffusion~\cite{fraclap} and the non-polynomial full-dispersion relation of Whitham's water-wave model~\cite{whitham}. KAN-based discovery methods~\cite{kanode,systemid1,sindykan,symbolickan,nonlocalkan} make learned functions inspectable but provide no operator architecture for this class.

Here we introduce the Kolmogorov--Arnold flow (KAFlow), an operator network that learns the generator of the target operator in a low-rank symbol space and integrates it exactly in time. The key insight is to represent each component of a physical operator where it is both sparse and low-rank: translation-invariant dispersion and dissipation as a spectral multiplier, position-dependent terms as a few separable products $a_r(\mat x)m_r(\bm\xi)$, and the pointwise nonlinearity as a gauge-fixed Kolmogorov--Arnold function. Our contributions are as follows:
\begin{itemize}
\item \textbf{Low-rank structure of position-dependent operators.} We show that finite-order variable-coefficient operators have resolution-independent separation rank, that the benchmark operators of KANO and the symbols KANO learns for them have rank two, and that a rank-$R$ symbol is applied exactly with $\mathcal O(Rm\log m)$ operations instead of $m^2$ symbol evaluations.
\item \textbf{An interpretable generator network.} KAFlow learns the generator in this low-rank symbol space with univariate Kolmogorov--Arnold functions, enforces identifiability by gauge fixing, integrates dynamics with exact exponential integrators and distills its closed form automatically, without human-designed templates.
\item \textbf{Validation on all KANO benchmarks and beyond.} KAFlow lowers the out-of-distribution loss of KANO by 3.4 to 11 times, recovers all coefficients to within $10^{-4}$, lowers the state infidelity of long-horizon quantum dynamics by up to 6 orders of magnitude, learns a variable-order fractional operator 15 times more accurately than KANO and FNO, and discovers nonlocal, fractional and non-polynomial dispersion relations on which sparse regression fails.
\end{itemize}

\section{MATERIALS AND METHODS}
\subsection{Experimental and technical design}
Figure~\ref{fig:1}A summarizes the design. KAFlow is trained on input--output pairs of an operator or on trajectories of an evolution equation, its closed form is distilled automatically, and both the network and the distilled closed form are evaluated on data outside the training distribution: unseen input families for static operators, extrapolated time steps for quantum dynamics and unseen initial conditions for nonlinear evolution equations. We use 4 groups of benchmarks. The first 2 are the complete benchmark suite of KANO~\cite{kano} with its official datasets: 3 position-dependent synthetic operators and 2 quantum Hamiltonians with 3 supervision types. The third is a variable-order fractional operator generated with the official KANO input generator. The fourth consists of 5 nonlinear evolution equations whose linear parts are local, fractional, nonlocal or non-polynomial. Baselines are the operator networks reported by KANO (FNO, U-FNO~\cite{ufno}, AM-FNO~\cite{amfno}, the pseudo-differential neural operator PDNO~\cite{pdno}, KANO and its variants) and, for evolution equations, a tuned FNO and 2 sparse-regression methods~\cite{pdefind,weaksindy}.

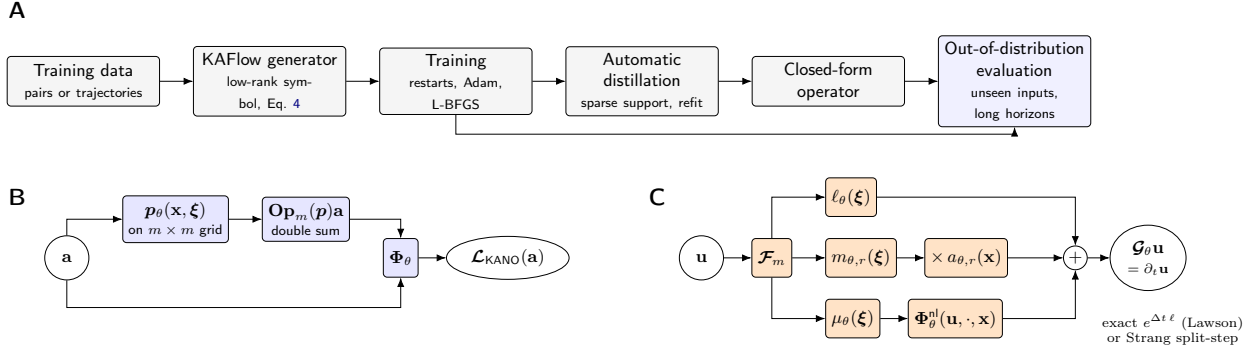
\begin{figure}[tbp]
\centering
\resizebox{\textwidth}{!}{%
\begin{tikzpicture}[font=\sffamily\small, >=Latex, node distance=5mm,
  st/.style={draw, rounded corners=2pt, fill=gray!8, align=center, minimum height=9mm, inner sep=3pt, text width=2.55cm},
  box/.style={draw, rounded corners=2pt, fill=orange!22, minimum height=7mm, inner sep=3pt},
  kbox/.style={draw, rounded corners=2pt, fill=blue!10, minimum height=7mm, inner sep=3pt},
  op/.style={draw, circle, inner sep=1pt, fill=white}]
\node[font=\sffamily\bfseries\large] at (-1.2,1.3) {A};
\node[st] (d) at (0,0) {Training data\\{\scriptsize pairs or trajectories}};
\node[st, right=6mm of d] (k) {KAFlow generator\\{\scriptsize low-rank symbol, Eq.~\ref{eq:kaflow}}};
\node[st, right=6mm of k] (t) {Training\\{\scriptsize restarts, Adam, L-BFGS}};
\node[st, right=6mm of t] (s) {Automatic distillation\\{\scriptsize sparse support, refit}};
\node[st, right=6mm of s] (c) {Closed-form operator};
\node[st, right=6mm of c, fill=blue!6] (e) {Out-of-distribution evaluation\\{\scriptsize unseen inputs, long horizons}};
\draw[->] (d) -- (k); \draw[->] (k) -- (t); \draw[->] (t) -- (s); \draw[->] (s) -- (c); \draw[->] (c) -- (e);
\draw[->] (t.south) -- ++(0,-0.35) -| (e.south);
\begin{scope}[yshift=-3.2cm]
\node[font=\sffamily\bfseries\large] at (-1.2,1.1) {B};
\node[draw, circle, minimum size=8mm] (a1) at (-0.3,0) {$\mat a$};
\node[kbox, right=6mm of a1, yshift=7mm, align=center] (p1) {$\boldsymbol p_\theta(\mat x,\bm\xi)$\\[-1pt]{\scriptsize on $m\times m$ grid}};
\node[kbox, right=6mm of p1, align=center] (kn) {$\Op(\boldsymbol p)\mat a$\\[-1pt]{\scriptsize double sum}};
\node[kbox, right=6mm of kn, yshift=-7mm] (phi1) {$\bm\Phi_\theta$};
\node[draw, ellipse, right=5mm of phi1] (o1) {$\ten L_{\text{KANO}}(\mat a)$};
\draw[->] (a1) |- (p1); \draw[->] (p1) -- (kn); \draw[->] (kn) -| (phi1);
\draw[->] (a1) -- ++(0,-0.9) -| (phi1); \draw[->] (phi1) -- (o1);
\end{scope}
\begin{scope}[xshift=11.2cm, yshift=-3.2cm]
\node[font=\sffamily\bfseries\large] at (-0.8,1.1) {C};
\node[draw, circle, minimum size=8mm] (a2) at (0,0) {$\mat u$};
\node[box, right=5mm of a2] (F) {$\ten F_m$};
\node[box, right=6mm of F, yshift=11mm] (l) {$\ell_\theta(\bm\xi)$};
\node[box, right=6mm of F] (m) {$m_{\theta,r}(\bm\xi)$};
\node[box, right=6mm of F, yshift=-11mm] (mu) {$\mu_\theta(\bm\xi)$};
\node[box, right=5mm of m] (ar) {$\times\,a_{\theta,r}(\mat x)$};
\node[box, right=5mm of mu] (ph) {$\bm\Phi^{\text{nl}}_\theta(\mat u,\cdot,\mat x)$};
\node[op, right=4mm of ar, xshift=6mm] (plus) {$+$};
\node[draw, ellipse, right=4mm of plus, align=center] (G) {$\ten G_\theta\mat u$\\[-1pt]{\scriptsize $=\partial_t\mat u$}};
\draw[->] (a2) -- (F); \draw[->] (F) |- (l); \draw[->] (F) -- (m); \draw[->] (F) |- (mu);
\draw[->] (m) -- (ar); \draw[->] (mu) -- (ph); \draw[->] (l) -| (plus); \draw[->] (ar) -- (plus); \draw[->] (ph) -| (plus); \draw[->] (plus) -- (G);
\node[font=\scriptsize, align=center] at ($(G)+(0.45,-1.3)$) {exact $e^{\Delta t\,\ell}$ (Lawson)\\ or Strang split-step};
\end{scope}
\end{tikzpicture}}
\caption{\textbf{Study design and architecture.} (A) KAFlow is trained on operator data, its closed form is distilled automatically, and both are evaluated outside the training distribution. (B) A KANO layer evaluates a dense symbol on the full position--frequency grid. (C) The KAFlow generator: a spectral multiplier $\ell$, $R$ separable position-dependent terms $a_r(\mat x)m_r(\bm\xi)$ and a gauge-fixed nonlinearity $\bm\Phi^{\text{nl}}$, all built from univariate functions and evaluated with fast Fourier transforms; for dynamics the generator is integrated exactly in time.}
\label{fig:1}
\end{figure}

\subsection{Related work}
\textit{Operator learning.} DeepONet~\cite{deeponet} learns the encoding and the latent map of an operator with 2 networks, and FNO~\cite{fno,no} fixes the encoding to a truncated Fourier transform with spectrally diagonal kernels; variants add local or multiscale kernels~\cite{ufno,amfno}. KAN-based operator networks replace multilayer perceptrons by KANs~\cite{deepokan,efkan}. KANO~\cite{kano} learns a dual-domain pseudo-differential symbol and recovers closed forms from its KAN edges.

\textit{Equation discovery.} Sparse identification of nonlinear dynamics (SINDy) and its PDE variants select terms from a library by sparse regression~\cite{sindy,pdefind,pysindy}, and weak formulations improve robustness~\cite{weaksindy}. PDE-Net~\cite{pdenet,pdenet2} learns moment-constrained convolutional stencils and a symbolic network. KAN-based approaches identify dynamical systems~\cite{kanode,systemid1,sindykan}, discrete symbolic structures~\cite{symbolickan} and nonlocal kernels of integro-differential equations in real space~\cite{nonlocalkan}.

\subsection{Pseudo-differential operators and the Kohn--Nirenberg quantization}
Operator learning approximates a mapping between function spaces, $\ten G:\mathcal A\to\mathcal U$, from pairs $\{(\mat a_i,\ten G(\mat a_i))\}_{i=1}^N$. On a periodic domain $D=(-\frac L2,\frac L2)^d$ with a uniform grid $\mathcal Y$ of spacing $h$ and retained modes $\Xi$, a pseudo-differential operator with symbol $\boldsymbol p(\mat x,\bm\xi)$ is discretized by the Kohn--Nirenberg quantization~\cite{kohnNirenberg,hormanderPDO}
\begin{equation}\label{eq:kn}
\Op(\boldsymbol p)\,\mat a(\mat x)\;=\;\Bigl(\frac hL\Bigr)^d\sum_{\bm\xi\in\Xi}\sum_{\mat y\in\mathcal Y} e^{i(\mat x-\mat y)\cdot\bm\xi}\,\boldsymbol p(\mat x,\bm\xi)\,\mat a(\mat y).
\end{equation}
A spatial multiplier $x$ and a spectral multiplier $i\xi$ are both symbols, so $\Op$ represents variable-coefficient differential operators exactly; for example, $\boldsymbol p(x,\xi)=x^2+\xi^2$ for $x^2-\partial_{xx}$. KANO stacks layers
\begin{equation}\label{eq:kano}
\ten L_{\text{KANO}}(\mat a)(\mat x)=\bm\Phi_{\theta_1}\bigl(\Op(\boldsymbol p_{\theta_2})\,\mat a(\mat x),\;\mat a(\mat x)\bigr),
\end{equation}
where the symbol $\boldsymbol p_{\theta_2}$ and the activation $\bm\Phi_{\theta_1}$ are compact KANs, whose edges are learnable univariate B-spline functions~\cite{kan,kan2}. The symbol is evaluated at every pair $(\mat x,\bm\xi)\in\mathcal Y\times\Xi$, and the reference implementation forms the phase tensor $e^{i(\mat x-\mat y)\cdot\bm\xi}$ for all triples. PDNO~\cite{pdno} first used this framework in operator learning: it assumes a single separable symbol $\boldsymbol p_{\mat x}(\mat x)\cdot\boldsymbol p_{\bm\xi}(\bm\xi)$ with multilayer-perceptron factors, acting on lifted latent channels between the lifting and projection networks of a neural operator~\cite{no}. The notation is summarized in Table~S1.

\subsection{Low-rank structure of position-dependent operators}
We call $\boldsymbol p$ of separation rank $R$ if $R$ is the smallest number with $\boldsymbol p(\mat x,\bm\xi)=\sum_{r=1}^R a_r(\mat x)m_r(\bm\xi)$; on the grid, it is the rank of the matrix $[\boldsymbol p(\mat x_j,\bm\xi_k)]_{jk}$.

\begin{proposition}[Resolution-independent rank]\label{prop:rank}
Let $\ten G=\sum_{k=0}^K c_k(\mat x)\,\partial^{\alpha_k}$ be a finite-order linear operator with $K+1$ coefficient functions, and let $\ten M(\bm\xi)$ be any translation-invariant, possibly nonlocal or fractional, multiplier. The symbol of $\ten M+\ten G$ has separation rank at most $K+2$, and its $m\times m$ grid matrix has rank at most $K+2$ for every resolution $m$.
\end{proposition}
The symbol is $\ten M(\bm\xi)\cdot1+\sum_k c_k(\mat x)(i\bm\xi)^{\alpha_k}$, a sum of $K+2$ outer products (Text~S1). All operators of the KANO benchmarks are of this form with rank 2: $\ten G_1$ has the symbol $x^2+\xi^2$, $\ten G_2$ and $\ten G_3$ have $ix\xi-\xi^2$, and the quantum Hamiltonians have $w(x)+\frac12\xi^2$. The singular values of these symbols on the benchmark grid vanish beyond the second ($\sigma_3/\sigma_1<10^{-15}$), and so do those of the symbols that KANO learns with its official checkpoints ($\sigma_3/\sigma_1\le3.5\times10^{-6}$): KANO evaluates a rank-2 object at $m^2=16{,}384$ points.

\subsection{Cost of the symbol calculus}
\begin{proposition}[Exact low-rank quantization]\label{prop:cost}
For $\boldsymbol p(\mat x,\bm\xi)=\ell(\bm\xi)+\sum_{r=1}^Ra_r(\mat x)m_r(\bm\xi)$,
\begin{equation}\label{eq:lowrank}
\Op(\boldsymbol p)\,\mat a\;=\;\ten F_m^{-1}\bigl[\ell\,\ten F_m\mat a\bigr]+\sum_{r=1}^R a_r(\mat x)\cdot\ten F_m^{-1}\bigl[m_r\,\ten F_m\mat a\bigr]
\end{equation}
holds exactly, where $\ten F_m$ is the discrete Fourier transform. Its evaluation costs $\mathcal O\bigl((R+1)\,m^d\log m^d\bigr)$ operations and $\mathcal O\bigl((R+1)\,m^d\bigr)$ symbol evaluations, whereas Eq.~\ref{eq:kn} with a general symbol requires $m^{2d}$ symbol evaluations and $\mathcal O(m^{2d})$ operations per input, or $\mathcal O(m^{3d})$ in the reference implementation of KANO.
\end{proposition}
In the Kohn--Nirenberg quantization the spatial argument of the symbol is applied after the inverse transform, so a product $a(\mat x)m(\bm\xi)$ is quantized to $a(\mat x)\cdot\ten F_m^{-1}[m\,\ten F_m\mat a]$ (Text~S1). During training the symbol must be re-evaluated at every step, so the $m^{2d}$ symbol evaluations dominate the cost of KANO at high resolution.

\subsection{Generators and propagators}
For an evolution $\partial_t\mat u=\ten G\mat u$, a static operator network must represent the propagator $e^{\Delta t\,\ten G}$, to which Proposition~\ref{prop:rank} does not apply. For the transport equation $u_t=c(x)u_x$, whose generator has rank 1, the propagator is the composition $f\mapsto f\circ\varphi_{\Delta t}$ with symbol $e^{i\xi(\varphi_{\Delta t}(x)-x)}$, whose rank grows with the displacement and the bandwidth: for $\varphi(x)=x+\varepsilon\sin x$ on the benchmark grid, the numerical rank at tolerance $10^{-4}$ is 11, 16 and 20 for $\varepsilon=0.1$, 0.2 and 0.3. Learning the generator keeps the representation in the low-rank class of Proposition~\ref{prop:rank} and delegates time evolution to an integrator whose stiff linear part is exact.

\subsection{Kolmogorov--Arnold flow architecture}
KAFlow represents the target operator by its generator (Fig.~\ref{fig:1}C):
\begin{equation}\label{eq:kaflow}
\ten G_\theta\mat u\;=\;\Op(\ell_\theta)\,\mat u\;+\;\sum_{r=1}^{R}a_{\theta,r}(\mat x)\cdot\Op(m_{\theta,r})\,\mat u\;+\;\bm\Phi^{\text{nl}}_\theta\bigl(\mat u,\;\Op(\mu_\theta)\mat u,\;\mat x\bigr),
\end{equation}
applied by Eq.~\ref{eq:lowrank}. For a static operator, $\ten G_\theta$ is the network output; for dynamics, $\partial_t\mat u=\ten G_\theta\mat u$, and the network output is the flow $\mat u(t+\Delta t)$.

\textit{Univariate Kolmogorov--Arnold components.} Every component of Eq.~\ref{eq:kaflow} is a univariate learnable function. The spectral symbols $\ell$, $m_r$ and $\mu$ are dictionary-augmented KAN edges,
\begin{equation}\label{eq:dict}
\ell(\xi)=s\Bigl[\sum_{p\in\mathcal P}c_p\Bigl(\frac{|\xi|}{\xi_{\text{ref}}}\Bigr)^{p}+\sum_{i}r_iB_i(\xi)\Bigr],\qquad \mathcal P=\{0,\tfrac12,1,\tfrac32,\dots,4\},
\end{equation}
with complex coefficients on the half spectrum $\xi\ge0$, so that the dictionary contains differential ($p\in\mathbb N$), fractional and Hilbert-type ($p\notin\mathbb N$, or a complex coefficient) multipliers, and the B-spline residual $r$ captures non-polynomial symbols. The residual is penalized with weights $w(\xi)=(E(\xi)/E_{\max}+\tau)^{-1}$, where $E$ is the training energy spectrum, so that it is only used where the data constrain it. The position functions $a_r(\mat x)$ are linear-plus-B-spline functions with zero mean, and the scales $s$ and $\xi_{\text{ref}}$ are fixed from data statistics (Text~S2). The nonlinearity is a separable KAN, $\bm\Phi(\mat z)=\sum_j\psi_j(z_j)+\sum_hw_h\prod_j\phi_{hj}(z_j)$, which contains products such as $u\,u_x$ and $x\,f$ as single terms.

\textit{Gauge fixing.} Eq.~\ref{eq:kaflow} admits redundant representations: a linear term can be carried by $\ell$, by $a_rm_r$ or by $\bm\Phi$. We remove this redundancy by construction. The position functions have zero mean, so constant coefficients belong to $\ell$, and the nonlinearity is projected onto its purely nonlinear part,
\begin{equation}\label{eq:gauge}
\bm\Phi^{\text{nl}}(\mat z,\mat x)=\bm\Phi(\mat z,\mat x)-\bm\Phi(\mat 0,\mat x)-\nabla_{\mat z}\bm\Phi(\mat 0,\mat x)\cdot\mat z,
\end{equation}
so that every linear term, including position-dependent ones, is carried by the symbol. Each term of the operator then has a unique place in the architecture, which makes automatic distillation possible.

\textit{Relation to PDNO.} Both PDNO~\cite{pdno} and KAFlow use separable symbols, but for different objects. PDNO applies a rank-1 symbol to lifted latent channels, so its factors are not the symbol of the physical operator. KAFlow acts on the physical field without lifting or projection, so that Eq.~\ref{eq:kaflow} is the symbol of the learned operator; it combines a translation-invariant multiplier with $R$ separable terms, which is exactly the class of Proposition~\ref{prop:rank}; its spectral factors contain differential, fractional and Hilbert-type multipliers by construction; and it learns a generator that is integrated in time rather than a static map.

\textit{Continuous-time flow.} For dynamics, KAFlow integrates $\partial_t\hat{\mat u}=\ell\,\hat{\mat u}+\ten N(\hat{\mat u})$ with the fourth-order Lawson integrating-factor Runge--Kutta scheme~\cite{lawson,kassam}, in which the stiff linear part $e^{\Delta t\,\ell}$ is exact, with $2/3$ dealiasing. For quantum dynamics, the generator is the Hamiltonian $\ten H=T(\xi)+v(x)+\kappa_1\rho+\kappa_2\rho^2$ with $\rho=|\psi|^2$, integrated by Strang splitting~\cite{strang}, $\psi\leftarrow e^{-\frac i2\Delta tV}\ten F^{-1}e^{-i\Delta tT}\ten F e^{-\frac i2\Delta tV}\psi$, which is unitary for every parameter value. The kinetic symbol $T(\xi)=s\sum_p\mathrm{softplus}(\theta_p)(|\xi|/\xi_{\text{ref}})^p$ is nonnegative, monotone and zero at $\xi=0$, which fixes the global-phase gauge. By Proposition~\ref{prop:cost}, a KAFlow evaluation costs $\mathcal O\bigl((R+J+1)\,m^d\log m^d\bigr)$ operations with $J$ nonlocal feature channels, and its parameters do not grow with $m$.

\subsection{Automatic symbolic distillation}
Because every component is a univariate function in a fixed gauge, the closed form of a trained KAFlow is obtained without human intervention. The learned position-dependent part $\sum_ra_r(\mat x)m_r(\bm\xi)$ is regressed on the monomials $\{(x^n-\langle x^n\rangle)(i\xi)^p\}$, the spectral symbol $\ell$ on the dictionary $\mathcal P$, and the nonlinearity on the products $\{f^2,f^3,ff_x,f^2f_x,ff_{xx},f_x^2,xf^2,xf^3\}$. Candidate supports are collected along a Bayesian information criterion path~\cite{bic} and from best-subset selection of each size, discarding ill-conditioned subsets (condition number above $10^3$). Every candidate is retrained in closed form, and the smallest support whose held-out loss on training-family data is within 1.1 times that of the best candidate is selected, the dense model being a candidate as well; the selected coefficients are refined by double-precision limited-memory Broyden--Fletcher--Goldfarb--Shanno (L-BFGS) optimization~\cite{lbfgs}. No symbolic template is specified by hand.

\subsection{Benchmarks, baselines and training}
\textit{Synthetic operators.} We use the official KANO datasets for $\ten G_1f=x^2f-f_{xx}$, $\ten G_2f=xf_x+f_{xx}$ and $\ten G_3f=f^3+xf_x+f_{xx}$ on 128 points: 2,000 training pairs from group A (periodic, chirped cosine and sine-beat inputs) and 400 test pairs from the unseen group B (sinc pulse, Gaussian--Hermite and wave-packet inputs). The metric is the mean relative $\ell_2$ loss on each group; the interpolation test evaluates inputs $f_\alpha=(1-\alpha)f_A+\alpha f_B$ with the ground-truth operator applied to $f_\alpha$. KAFlow uses $R=2$ and $J=1$ for all 3 operators; 3 random initializations are trained for 1,000 Adam~\cite{adam} iterations, the one with the lowest held-out group A loss is trained further and refined by full-batch L-BFGS.

\textit{Quantum dynamics.} We use the official double-well (DW) and nonlinear Schr\"odinger equation (NLSE) datasets of KANO: $i\partial_t\psi=-\frac12\psi_{xx}+w(x)\psi$ and $i\partial_t\psi=-\frac12\psi_{xx}+w(x)\psi+|\psi|^2\psi$ with $w(x)=x^4-(x-\frac1{32})^2+0.296$, 200 initial states and 100 snapshots every 0.1~ms, of which the first 10 are used for training. The supervision types are full wave functions, position and momentum probability mass functions (pos~\&~mom) and position probability mass functions only (pos), with the losses of KANO. The metric is the state infidelity $1-|\langle\tilde\psi,\psi\rangle|$ over all 100 steps.

\textit{Variable-order fractional operator.} We define $\ten G_4f=\Op(|\xi|^{1+0.5\sin x})f$, whose order varies between $\frac12$ and $\frac32$ in space and whose numerical separation rank is 4 at tolerance $10^{-4}$. Inputs are drawn from the official KANO generator with the same groups, and outputs are computed with the exact Kohn--Nirenberg quantization. KANO, KANO with a multilayer perceptron (MLP) symbol (the first stage of its official recipe) and FNO are trained with the official recipes of KANO.

\textit{Nonlocal and dispersive dynamics.} We generate 256 training and 64 test trajectories of 100 steps from random initial conditions on 128 points for the Burgers equation $u_t=0.02u_{xx}-(\frac12u^2)_x$, the fractional Burgers equation $u_t=-0.05(-\Delta)^{3/4}u-(\frac12u^2)_x$, the Benjamin--Ono equation $u_t=-\ten Hu_{xx}-(\frac12u^2)_x$ with the Hilbert transform $\ten H$, the Korteweg--de Vries equation $u_t=-u_{xxx}-3(u^2)_x$~\cite{kdv} and the Whitham equation $u_t=-\partial_x\ten Ku-\frac34(u^2)_x$ with $\widehat{\ten K}(\xi)=\sqrt{\tanh\xi/\xi}$. The Korteweg--de Vries and Whitham equations were held out during the design of KAFlow, with the training and model-selection rules fixed beforehand. KAFlow is trained with a rollout curriculum and selects its sparse support on the full-horizon validation rollout. FNO is tuned on the same validation trajectories, and PDE-FIND~\cite{pdefind} and Weak SINDy~\cite{weaksindy} (PySINDy~\cite{pysindy}) use a library of $\{u,u^2,u^3\}\times\{\partial_x^0,\dots,\partial_x^4\}$ with thresholds selected by validation rollouts. All details are given in Text~S2.

\subsection{Statistical analysis}
KAFlow results are means over 6 random seeds for the synthetic operators and evolution equations and over 3 seeds for the quantum benchmarks and $\ten G_4$; baselines trained here use 3 seeds. Uncertainties are standard deviations over seeds. Baseline numbers for U-FNO, AM-FNO, PDNO, KANO and its variants on the KANO benchmarks are those reported by KANO~\cite{kano}; we reproduced the FNO, KANO and KANO-symbolic results with the official checkpoints. Coefficients of closed forms are reported for the model with the lowest validation loss, and the number of seeds that select the same support is stated.

\section{RESULTS}
\subsection{Generalization on position-dependent synthetic operators}
KAFlow attains the lowest unseen-family loss among all neural models on every synthetic operator (Table~\ref{tab:loss} and Fig.~\ref{fig:2}A to C), 3.4, 11 and 7.3 times lower than KANO and 2 to 3 orders of magnitude lower than the FNO family. Its losses on the unseen families are even lower than on the training families (B/A $<1$), and its loss ratio stays below 1 along the whole interpolation path from group A to group B, whereas that of FNO increases by up to 14.5 times (Fig.~\ref{fig:2}D to F): KAFlow converges to the operator itself rather than to an in-sample mapping. Removing the separable position-dependent terms and letting the nonlinearity $\bm\Phi(\cdot,\mat x)$ carry the position dependence, with the same training protocol, raises the loss to the level of KANO, 3.6, 21 and 19 times above that of the full model (Table~\ref{tab:loss}, KAFlow without separable terms): the low-rank symbol is the component that lowers the loss below KANO. Against PDNO, which uses a single separable symbol on latent channels, the unseen-family loss of KAFlow is 15 to 270 times lower.

\begin{table}[tbp]
\centering
\caption{\textbf{Relative $\ell_2$ losses ($\times10^{-4}$) and parameter counts on the synthetic operators.} A, training families; B, unseen families. Baselines as reported by KANO; FNO, KANO and KANO-symbolic were reproduced with the official checkpoints. KAFlow: mean of 6 seeds. The last row is the exact operator evaluated against the stored single-precision labels.}
\label{tab:loss}
\small
\setlength{\tabcolsep}{4pt}
\resizebox{\textwidth}{!}{%
\begin{tabular}{l ccc ccc ccc}
\toprule
& \multicolumn{3}{c}{$\ten G_1$} & \multicolumn{3}{c}{$\ten G_2$} & \multicolumn{3}{c}{$\ten G_3$}\\
\cmidrule(lr){2-4}\cmidrule(lr){5-7}\cmidrule(lr){8-10}
Model (parameters) & A & B & B/A & A & B & B/A & A & B & B/A\\
\midrule
FNO (566k) & 6.36 & 98.8 & 15.53 & 10.6 & 87.0 & 8.21 & 11.4 & 81.4 & 7.14\\
U-FNO (579k) & 2.79 & 22.9 & 8.21 & 8.14 & 339 & 41.65 & 92.4 & 292 & 3.16\\
AM-FNO (548k) & 1.08 & 20.9 & 19.35 & 1.20 & 16.5 & 13.75 & 1.16 & 29.8 & 25.69\\
PDNO (538k) & 1.41 & 6.31 & 4.5 & 1.92 & 12.1 & 6.3 & 4.03 & 27.2 & 6.7\\
KANO (152) & 1.04 & 1.44 & 1.38 & 0.629 & 0.749 & 1.19 & 0.716 & 0.737 & 1.03\\
KANO, MLP symbol (2k) & 3.37 & 6.59 & 1.96 & 4.49 & 8.07 & 1.80 & 3.59 & 6.87 & 1.91\\
KANO-symbolic & 0.512 & 0.526 & 1.03 & 0.498 & 0.500 & 1.00 & 0.520 & 0.536 & 1.03\\
\midrule
KAFlow without separable terms (216) & 1.48 & 1.53 & 1.03 & 1.88 & 1.38 & 0.74 & 2.28 & 1.96 & 0.86\\
KAFlow (336) & 0.632 & \textbf{0.422} & 0.67 & 0.119 & \textbf{0.067} & 0.56 & 0.168 & \textbf{0.101} & 0.60\\
KAFlow-symbolic & 0.0075 & \textbf{0.0125} & 1.66 & 0.0070 & \textbf{0.0115} & 1.63 & 0.0086 & \textbf{0.0133} & 1.55\\
\midrule
Exact operator & 0.0081 & 0.0129 & 1.60 & 0.0076 & 0.0119 & 1.57 & 0.0083 & 0.0128 & 1.54\\
\bottomrule
\end{tabular}}
\end{table}

\begin{figure}[tbp]
\centering
\includegraphics[width=\textwidth]{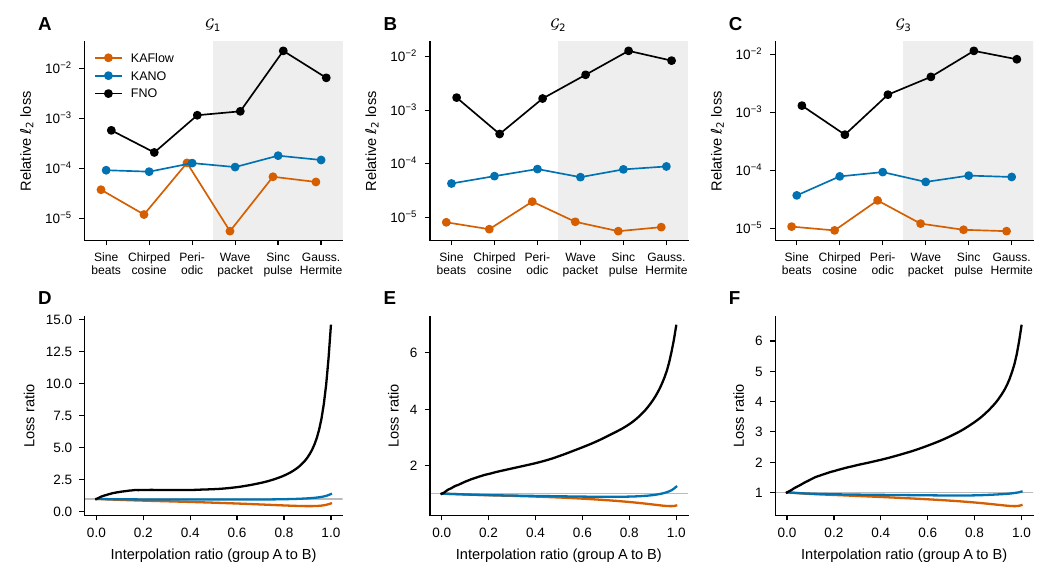}
\caption{\textbf{Generalization on the synthetic operators.} (A to C) Relative $\ell_2$ loss per input family for $\ten G_1$, $\ten G_2$ and $\ten G_3$; the shaded families (group B) are unseen during training. KANO and FNO use the official checkpoints; KAFlow is the mean of 6 seeds. (D to F) Loss ratio along the interpolation from group A to group B inputs, relative to the loss on group A.}
\label{fig:2}
\end{figure}

\subsection{Automatic recovery of closed-form operators}
The learned components of KAFlow reproduce each term of the operators in its own place of the architecture, as Eq.~\ref{eq:gauge} enforces (Fig.~\ref{fig:3}): the position profile $x^2$, the dissipative symbols $\pm\xi^2$, the position-dependent advection $x\,\xi$ and the cubic nonlinearity, whose least-squares fit on unseen inputs is $1.000f^3-0.000f$. The automatic distillation selects the exact support for all 18 trained models (6 seeds $\times$ 3 operators) and removes the neural nonlinearity for $\ten G_1$ and $\ten G_2$. The resulting KAFlow-symbolic recovers every coefficient to within $10^{-5}$ for $\ten G_1$ and $\ten G_2$ and to within $8.4\times10^{-5}$ for $\ten G_3$, with all other coefficients below $10^{-5}$ (Table~\ref{tab:evidence}). Its losses coincide with those of the exact operator evaluated against the single-precision labels (Table~\ref{tab:loss}, last row), so KAFlow-symbolic reaches the precision floor of the benchmark. The KANO-symbolic model is obtained by retraining human-specified templates and retains spurious $f^2$, $f$ and constant terms (Table~\ref{tab:evidence}).

\begin{figure}[tbp]
\centering
\includegraphics[width=\textwidth]{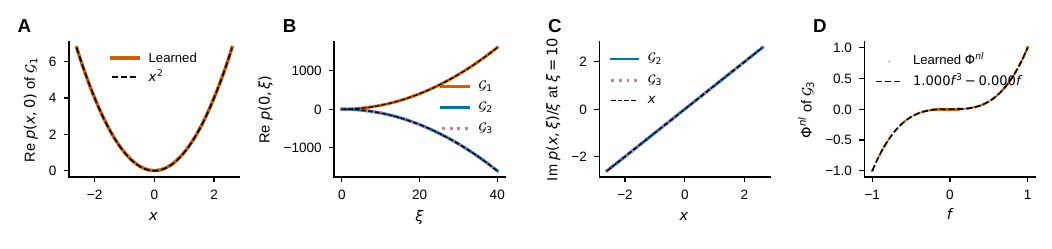}
\caption{\textbf{Learned components of KAFlow} (solid) against the ground truth (dashed). (A) $\mathrm{Re}\,\boldsymbol p(x,0)$ of $\ten G_1$. (B) $\mathrm{Re}\,\boldsymbol p(0,\xi)$ of $\ten G_1$, $\ten G_2$ and $\ten G_3$. (C) $\mathrm{Im}\,\boldsymbol p(x,\xi)/\xi$ of $\ten G_2$ and $\ten G_3$ at $\xi=10$. (D) Gauge-fixed nonlinearity $\bm\Phi^{\text{nl}}$ of $\ten G_3$ on unseen inputs. Positions are shown where the input envelope equals 1.}
\label{fig:3}
\end{figure}

\begin{table}[tbp]
\centering
\caption{\textbf{Ground-truth and learned operators} (coefficients rounded to the fourth decimal place). KANO: as reported, from human-designed templates. KAFlow: automatic distillation, seed 0; all 6 seeds select the same support.}
\label{tab:evidence}
\small
\resizebox{\textwidth}{!}{%
\begin{tabular}{@{}l l l@{}}
\toprule
Ground truth & Learned KANO operator & Learned KAFlow operator (automatic)\\
\midrule
$x^2 f-\partial_{xx}f$ & $(x^2+0.0003) f-\partial_{xx}f$ & $1.0000\,x^2 f-1.0000\,\partial_{xx}f$\\
$x\partial_xf+\partial_{xx}f$ & $0.9996\,x\partial_xf+\partial_{xx}f-0.0003\,f$ & $1.0000\,x\partial_xf+1.0000\,\partial_{xx}f$\\
$f^3+x\partial_xf+\partial_{xx}f$ & \makecell[l]{$1.0001\,f^3+0.99997\,x\partial_xf+0.99997\,\partial_{xx}f$\\ $\;-0.0002\,f^2-0.0003\,f-0.0001$} & $1.0000\,f^3+1.0000\,x\partial_xf+1.0000\,\partial_{xx}f$\\
\bottomrule
\end{tabular}}
\end{table}

\subsection{Long-horizon quantum dynamics}
With full supervision, KAFlow reaches state infidelities of $2.3\times10^{-10}$ on DW and $7.4\times10^{-12}$ on NLSE after 90 extrapolated steps, 4 to 6 orders of magnitude below Q-KANO, the quantum variant of KANO, and 87 and 2,700 times below its symbolic version (Table~\ref{tab:fidelity}). With the physically attainable pos~\&~mom supervision, KAFlow reaches $1.9\times10^{-7}$ and $1.8\times10^{-7}$, about 30 times lower than Q-KANO, and with pos supervision alone it reaches $1.6\times10^{-4}$ and $1.7\times10^{-4}$, 29 and 350 times lower than Q-KANO. The infidelity grows slowly over the 90 extrapolated steps (Fig.~\ref{fig:4}), as the unitary Strang flow of the learned generator is iterated with 25 (DW) and 27 (NLSE) parameters.

\begin{table}[tbp]
\centering
\caption{\textbf{State infidelity after 90 additional time-evolution steps.} Baselines as reported by KANO. KAFlow: mean of 3 seeds.}
\label{tab:fidelity}
\small
\begin{tabular}{lcc}
\toprule
Model and supervision & DW & NLSE\\
\midrule
FNO (full) & $1.5\times10^{-2}$ & $1.6\times10^{-2}$\\
Q-KANO (full) & $6.3\times10^{-6}$ & $6.8\times10^{-6}$\\
Q-KANO (pos \& mom) & $6.3\times10^{-6}$ & $5.6\times10^{-6}$\\
Q-KANO (pos) & $4.7\times10^{-3}$ & $6.1\times10^{-2}$\\
Q-KANO, MLP symbol (pos \& mom) & $7.7\times10^{-6}$ & $8.5\times10^{-6}$\\
Q-KANO-symbolic (full) & $2.0\times10^{-8}$ & $2.0\times10^{-8}$\\
Q-KANO-symbolic (pos \& mom) & $2.0\times10^{-8}$ & $3.0\times10^{-8}$\\
\midrule
KAFlow (full) & $\mathbf{2.3\times10^{-10}}$ & $\mathbf{7.4\times10^{-12}}$\\
KAFlow (pos \& mom) & $1.9\times10^{-7}$ & $1.8\times10^{-7}$\\
KAFlow (pos) & $1.6\times10^{-4}$ & $1.7\times10^{-4}$\\
\bottomrule
\end{tabular}
\end{table}

\begin{figure}[tbp]
\centering
\includegraphics[width=\textwidth]{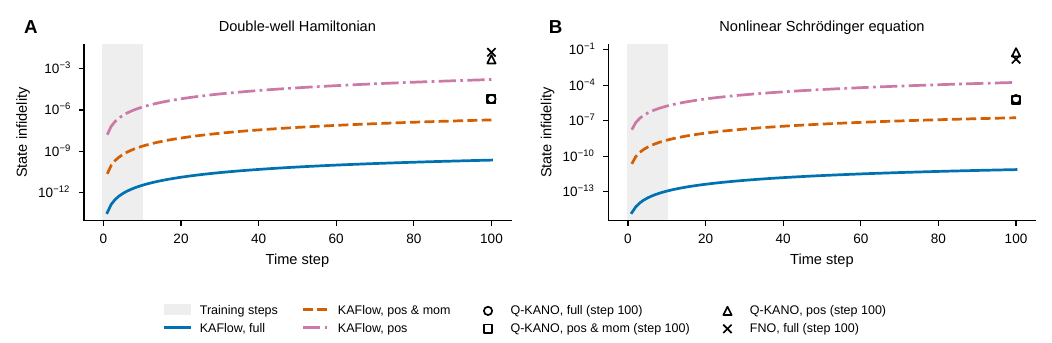}
\caption{\textbf{State infidelity of long-horizon quantum dynamics} over 100 time steps (mean of 3 seeds). (A) Double-well Hamiltonian. (B) Nonlinear Schr\"odinger equation. The first 10 steps (shaded) are used for training. Open markers and crosses show the step-100 infidelities reported for Q-KANO and FNO.}
\label{fig:4}
\end{figure}

The learned Hamiltonians are compared in Table~\ref{tab:hamiltonian}. With full supervision, KAFlow recovers the quartic, quadratic and constant coefficients of the potential and the kinetic coefficient to the fourth decimal place, and the distilled NLSE model recovers the Kerr coefficient as 1.0000, compared to 0.9815 and a spurious $0.0110\,|\psi|$ term for Q-KANO. With pos~\&~mom supervision, where the Kerr coefficient enters only through higher-order correlations of the probability mass functions, the distilled KAFlow recovers $\kappa_1=0.959\pm0.001$ (3 seeds), compared to 0.9514 for Q-KANO.

\begin{table}[tbp]
\centering
\caption{\textbf{Ground-truth and learned Hamiltonian symbols with full supervision} (coefficients rounded to the fourth decimal place).}
\label{tab:hamiltonian}
\small
\resizebox{\textwidth}{!}{%
\begin{tabular}{@{}l l l@{}}
\toprule
Hamiltonian & Model & Learned symbolic structure\\
\midrule
DW & Ground truth & $x^4-x^2+0.0625\,x+0.2950+0.5\,\xi^2$\\
 & Q-KANO & $1.0004\,x^4+0.0001\,x^3-1.0013\,x^2+0.0625\,x+0.2955+0.5\,\xi^2$\\
 & KAFlow & $1.0000\,x^4-0.0001\,x^3-1.0000\,x^2+0.0627\,x+0.2950+0.5\,\xi^2$\\
\midrule
NLSE & Ground truth & $x^4-x^2+0.0625\,x+0.2950+0.5\,\xi^2+|\psi|^2$\\
 & Q-KANO & $1.0005\,x^4-0.0001\,x^3-1.0014\,x^2+0.0626\,x+0.2942+0.5\,\xi^2+0.9815\,|\psi|^2+0.0110\,|\psi|$\\
 & KAFlow & $1.0000\,x^4-1.0000\,x^2+0.0625\,x+0.2950+0.5\,\xi^2+1.0000\,|\psi|^2$\\
\bottomrule
\end{tabular}}
\end{table}

\subsection{Variable-order fractional operator}
Beyond the polynomial symbols of the KANO benchmarks, KAFlow learns the variable-order fractional operator $\ten G_4$, whose symbol is neither polynomial in $x$ nor in $\xi$. With $R=4$ separable terms, KAFlow attains $2.9\times10^{-3}$ on the unseen families, 17 times lower than KANO with an MLP symbol and 15 times lower than FNO, with 0.08\% of the FNO parameters (Table~\ref{tab:h1} and Fig.~\ref{fig:5}). The learned symbol reproduces the spatially varying order across the energetic band (Fig.~\ref{fig:5}A and B). The published KANO configuration, whose symbol KAN has an additive real part, cannot represent this symbol, and KANO with an MLP symbol generalizes 9 times worse than it fits the training families. Four separable terms suffice, consistent with the numerical rank of the symbol.

\begin{table}[tbp]
\centering
\caption{\textbf{Relative $\ell_2$ losses ($\times10^{-2}$) on the variable-order fractional operator $\ten G_4$} (mean of 3 seeds).}
\label{tab:h1}
\small
\begin{tabular}{lccc}
\toprule
Model (parameters) & A & B & B/A\\
\midrule
FNO (566k) & 0.085 & 4.44 & 52.2\\
KANO, published configuration & 51.6 & 206 & 3.99\\
KANO, MLP symbol (1.2k) & 0.557 & 5.06 & 9.08\\
KAFlow, $R=2$ (336) & 0.193 & 0.886 & 4.59\\
KAFlow, $R=4$ (456) & 0.124 & \textbf{0.293} & 2.36\\
\bottomrule
\end{tabular}
\end{table}

\begin{figure}[tbp]
\centering
\includegraphics[width=\textwidth]{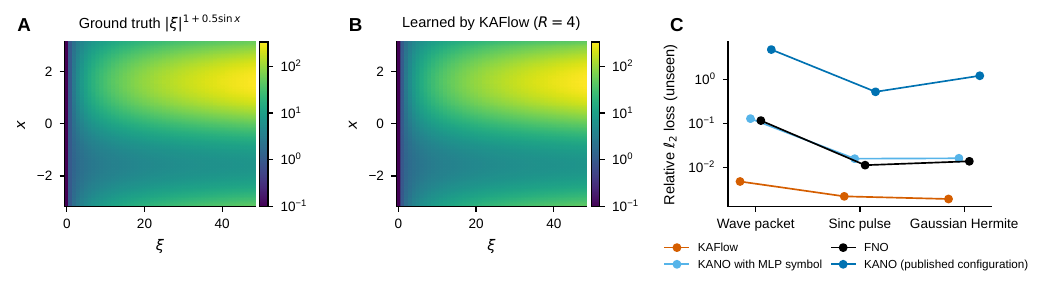}
\caption{\textbf{Variable-order fractional operator $\ten G_4$.} (A) Ground-truth symbol $|\xi|^{1+0.5\sin x}$. (B) Symbol learned by KAFlow with $R=4$. (C) Relative $\ell_2$ loss on the unseen families (mean of 3 seeds).}
\label{fig:5}
\end{figure}

\subsection{Nonlocal and dispersive dynamics}
KAFlow attains the lowest 100-step loss on all 5 evolution equations with 302 parameters, against 287,000 to 550,000 for FNO (Table~\ref{tab:dynamics} and Fig.~\ref{fig:6}A to E). The gap to FNO is 3.0 times on the Burgers equation, 3.3 times on the fractional Burgers equation and 4.3 times on the Whitham equation, whose full-dispersion relation is non-polynomial. Sparse regression succeeds on the local Burgers equation but fails on the nonlocal and fractional equations, whose linear parts lie outside any local polynomial library: KAFlow is 20 (fractional Burgers), 220 (Benjamin--Ono) and 7.4 (Whitham) times more accurate than the best sparse-regression method. Its error grows slowly with the rollout horizon, as dispersion and dissipation are integrated exactly.

The learned linear symbols are shown in Table~\ref{tab:dynamics} (bottom) and Fig.~\ref{fig:6}F to J. The automatic distillation selects the exact closed form for the Burgers, fractional Burgers, Benjamin--Ono and Korteweg--de Vries equations, recovering the fractional order $\frac32$, the Hilbert-type symbol $-i\xi|\xi|$ and the dispersion $i\xi^3$ with coefficient errors of 0.02\% to 0.6\%, including the held-out Korteweg--de Vries equation. For the Whitham equation, the model keeps the dense symbol and recovers $-\sqrt{\xi\tanh\xi}$ to 0.25\% over the resolved band.

\begin{table}[tbp]
\centering
\caption{\textbf{Nonlocal and dispersive dynamics.} Top, relative $\ell_2$ loss after 100 rollout steps on unseen initial conditions (FNO, 3 seeds; KAFlow, 6 seeds, for the Korteweg--de Vries and Benjamin--Ono equations the 5 that converged; mean $\pm$ standard deviation). Bottom, linear symbols learned by the KAFlow model with the lowest validation loss. BO, Benjamin--Ono; KdV, Korteweg--de Vries.}
\label{tab:dynamics}
\small
\resizebox{\textwidth}{!}{%
\begin{tabular}{lccccc}
\toprule
Model (parameters) & Burgers & Fractional Burgers & BO & KdV & Whitham\\
\midrule
PDE-FIND & $6.7\times10^{-3}$ & $9.6\times10^{-2}$ & $7.5\times10^{-1}$ & $9.2\times10^{-1}$ & diverged\\
Weak SINDy & $2.6\times10^{-2}$ & $1.0\times10^{-1}$ & $6.5\times10^{-1}$ & $5.9\times10^{-2}$ & $7.6\times10^{-1}$\\
FNO (287k to 550k) & $(1.30\pm0.03)\times10^{-2}$ & $(1.59\pm0.10)\times10^{-2}$ & $(3.12\pm0.36)\times10^{-3}$ & $(1.98\pm1.3)\times10^{-2}$ & $(4.42\pm0.24)\times10^{-1}$\\
KAFlow (302) & $\mathbf{(4.29\pm2.0)\times10^{-3}}$ & $\mathbf{(4.88\pm4.0)\times10^{-3}}$ & $\mathbf{(2.90\pm1.9)\times10^{-3}}$ & $\mathbf{(1.61\pm1.1)\times10^{-2}}$ & $\mathbf{(1.03\pm0.27)\times10^{-1}}$\\
\midrule
Ground-truth $\ell(\xi)$ & $-0.02\,\xi^2$ & $-0.05\,|\xi|^{3/2}$ & $-i\,\xi|\xi|$ & $i\,\xi^3$ & $-i\sqrt{\xi\tanh\xi}$\\
Learned $\ell(\xi)$ (automatic) & $-0.0201\,\xi^2$ & $-0.0503\,|\xi|^{3/2}$ & $-1.0002\,i\,\xi|\xi|$ & $0.9994\,i\,\xi^3$ & spline, relative error $2.5\times10^{-3}$\\
\bottomrule
\end{tabular}}
\end{table}

\begin{figure}[tbp]
\centering
\includegraphics[width=\textwidth]{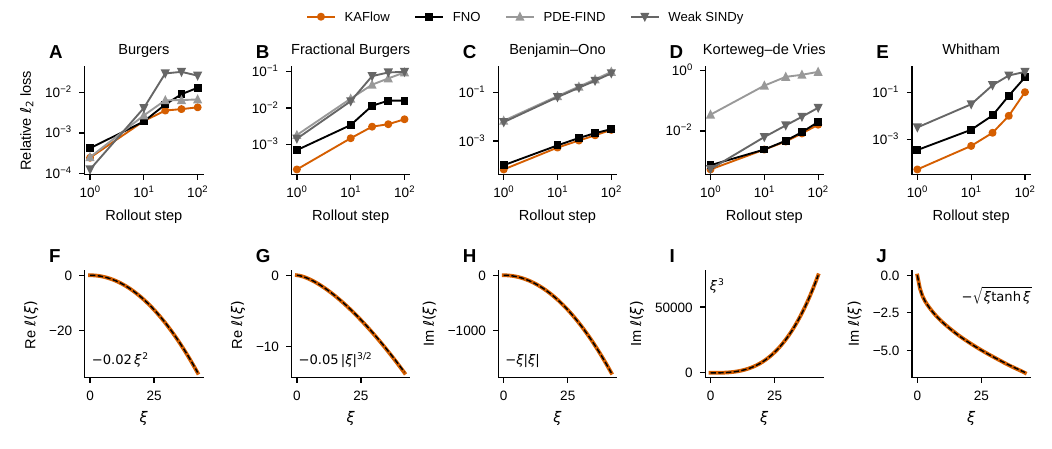}
\caption{\textbf{Nonlocal and dispersive dynamics.} (A to E) Relative $\ell_2$ loss on unseen initial conditions over 100 rollout steps; PDE-FIND diverges on the Whitham equation. (F to J) Linear symbols $\ell(\xi)$ learned by KAFlow (solid) against the ground truth (dashed, formula in each panel) over the resolved band; (F to I) are the closed forms selected automatically and (J) is the non-polynomial Whitham symbol represented by the spline residual.}
\label{fig:6}
\end{figure}

\subsection{Computational cost}
Proposition~\ref{prop:cost} predicts that the cost of a training step of KANO grows with the $m^2$ symbol evaluations and that of KAFlow with $m\log m$. At the benchmark resolution $m=128$, a training step (batch of 64, forward and backward pass) takes 10.1~ms for KANO and 8.5~ms for KAFlow on a central processing unit. When the resolution grows from 64 to 1,024, the forward pass with symbol evaluation grows from 3.4~ms to 3.6~s for KANO and from 2.1 to 5.2~ms for KAFlow, a factor of 690 at $m=1{,}024$ (Table~\ref{tab:cost}).

\begin{table}[tbp]
\centering
\caption{\textbf{Forward time (ms) with symbol evaluation} (batch of 16; Intel Xeon Silver 4510, 8 threads). KANO, $\mathcal O(m^2)$: the same quantization implemented as a Fourier transform followed by a dense position--frequency sum, which reproduces the official output to a relative difference of $6.8\times10^{-6}$.}
\label{tab:cost}
\small
\begin{tabular}{lccccc}
\toprule
Resolution $m$ & 64 & 128 & 256 & 512 & 1,024\\
\midrule
KANO, official implementation & 3.4 & 8.2 & 53.8 & 505 & --\\
KANO, $\mathcal O(m^2)$ implementation & 3.4 & 8.0 & 39.3 & 356 & 3,596\\
KAFlow & \textbf{2.1} & \textbf{2.5} & \textbf{2.9} & \textbf{3.7} & \textbf{5.2}\\
\bottomrule
\end{tabular}
\end{table}

\section{DISCUSSION}
Operator networks that are meant to stand in for physical laws need two things that are usually traded against each other: generalization outside the training distribution and a closed form that can be read out. KANO showed that a dual-domain symbol provides both for position-dependent dynamics. Our results show that the structure of the physical operators themselves allows a much more compact and more accurate representation. Position-dependent operators have low separation rank, and a symbol built from a translation-invariant multiplier and a few separable terms represents them exactly, applies them with fast Fourier transforms and leaves no room for redundant representations once the gauge of its univariate components is fixed. On every benchmark of KANO, this representation lowers the out-of-distribution loss, and the ablation identifies the separable terms as the component that does so.

Two design choices carry over to other scientific-modeling problems. First, learning the generator instead of the propagator keeps the representation in the low-rank class and moves the time evolution into an integrator whose stiff linear part is exact; this is what keeps the quantum infidelity at the $10^{-10}$ level over 90 extrapolated steps and the dispersive rollouts stable over 100 steps. Second, fixing the gauge of the components turns symbolic regression into an automatic step: every term has one place in the architecture, so a sparse fit with a parsimony rule recovers the exact closed form for all synthetic operators, the potential and Kerr coefficient of the Schr\"odinger equation, and fractional, Hilbert-type and dispersive symbols that lie outside the libraries of sparse regression. The same structure, a translation-invariant multiplier plus a few separable position-dependent terms, describes many operators of mathematical physics, including variable-coefficient diffusion, transport and Schr\"odinger operators. Extending the low-rank generator to two- and three-dimensional systems, where the dense symbol calculus is most costly, is a natural direction for future work.

\section*{Acknowledgments}
\subsection*{Author Contributions}
M. Cheon conceived the study, developed the method, ran the experiments, analyzed the results and wrote the manuscript.

\subsection*{Funding}
The author acknowledges that they did not receive funding for this work.

\subsection*{Conflicts of Interest}
The author declares that there is no conflict of interest regarding the publication of this article.

\subsection*{Data Availability}
The synthetic-operator and quantum datasets and checkpoints of KANO are available from the official KANO repository (\url{https://github.com/jinucsb/KANO}). The data generators for the variable-order fractional operator and the evolution equations, the trained KAFlow models and the code for training, distillation, baselines, analyses and figures are available from the corresponding author upon reasonable request.

\section*{Supplementary Materials}
Texts~S1 and S2\\
Table~S1

\bibliographystyle{vancouver}
\bibliography{main}

\clearpage
\setcounter{figure}{0}
\setcounter{table}{0}
\renewcommand{\thefigure}{S\arabic{figure}}
\renewcommand{\thetable}{S\arabic{table}}
\begin{center}
{\large\bfseries Supplementary Materials}
\end{center}

\subsection*{Text S1. Proofs}
\textit{Proof of Proposition~\ref{prop:rank}.} The symbol of $\partial^{\alpha}$ is $(i\bm\xi)^{\alpha}$ and that of multiplication by $c(\mat x)$ is $c(\mat x)$; in the Kohn--Nirenberg quantization the symbol of $c(\mat x)\partial^\alpha$ is the product $c(\mat x)(i\bm\xi)^\alpha$. Hence the symbol of $\ten M+\ten G$ is $\ten M(\bm\xi)\cdot1+\sum_{k=0}^Kc_k(\mat x)(i\bm\xi)^{\alpha_k}$, a sum of $K+2$ products of a function of $\mat x$ and a function of $\bm\xi$. Sampling on $\mathcal Y\times\Xi$ gives a sum of $K+2$ rank-1 matrices, whose rank is at most $K+2$ for every $m$.

\textit{Proof of Proposition~\ref{prop:cost}.} By Eq.~\ref{eq:kn} and linearity in $\boldsymbol p$, it suffices to consider $\boldsymbol p=a(\mat x)m(\bm\xi)$: $\Op(am)\,\mat u(\mat x)=a(\mat x)\,(h/L)^d\sum_{\bm\xi}e^{i\mat x\cdot\bm\xi}m(\bm\xi)\sum_{\mat y}e^{-i\mat y\cdot\bm\xi}\mat u(\mat y)=a(\mat x)\cdot\ten F_m^{-1}[m\,\ten F_m\mat u](\mat x)$. Each term costs one forward and one inverse fast Fourier transform, $\mathcal O(m^d\log m^d)$, one pointwise product and $m^d$ evaluations of $m$ and $a$. For a general symbol, the inner sum over $\mat y$ is a Fourier transform, but the outer sum couples every $\mat x$ with every $\bm\xi$ through $\boldsymbol p(\mat x,\bm\xi)$, which requires $m^{2d}$ symbol values and $\mathcal O(m^{2d})$ operations; forming the phase $e^{i(\mat x-\mat y)\cdot\bm\xi}$ for all triples, as in the reference implementation, costs $\mathcal O(m^{3d})$.

\textit{Exactness of the linear flow.} For $\partial_t\hat{\mat u}=\ell\,\hat{\mat u}+\ten N(\hat{\mat u})$ with diagonal $\ell$, the Lawson scheme applies the Runge--Kutta method to $\hat{\mat v}=e^{-t\ell}\hat{\mat u}$. If $\ten N\equiv0$, it returns $e^{\Delta t\,\ell}\hat{\mat u}$ exactly for every $\Delta t$, so dissipation and dispersion of any order, including fractional and non-polynomial symbols, introduce no time-discretization error and no step-size restriction. Each factor of the Strang step is unitary for real $V$ and $T$, so the quantum flow preserves the norm for every parameter value.

\subsection*{Text S2. Experiment details}
\textit{Common configuration.} All spectral symbols are dictionary-augmented functions (Eq.~\ref{eq:dict}) on the real-Fourier half spectrum, with 16 cubic B-spline intervals for the residual and $\tau=10^{-3}$. The scale $s$ is the ratio of output to input root-mean-square values for static operators and the relative rate of change $\mathrm{rms}(u_{t+1}-u_t)/(\Delta t\,\mathrm{rms}(u))$ for dynamics; $\xi_{\text{ref}}$ is the energy-weighted root-mean-square frequency of the training data. Position functions use 12 cubic B-spline intervals plus a linear term and are centered to zero mean. The nonlinearity is a separable KAN with 2 product terms and 8 B-spline intervals per factor. The penalties are an $\ell_1$ penalty on dictionary coefficients ($\lambda_1=10^{-5}$, three-fold for half-integer exponents) and the energy-weighted residual penalty ($\lambda_r=10^{-3}$); for dynamics both are multiplied by the one-step loss of the persistence predictor. Non-finite losses or gradients restore the last finite checkpoint and halve the learning rate.

\textit{Synthetic operators.} The official datasets have $N=128$ points on $(-\pi,\pi]$, 2,000 training pairs from group A, 400 held-out pairs from group A for model selection and 400 test pairs from group B. Three random initializations are trained for 1,000 Adam iterations (batch of 64, learning rate from $10^{-2}$ to $3\times10^{-3}$); the one with the lowest held-out group A loss is trained for 7,000 further iterations (cosine schedule from $5\times10^{-3}$); sparsification candidates are fine-tuned for 800 iterations each; the selected model is refined by full-batch L-BFGS (800 iterations). The interpolation test pairs the $i$-th held-out group A sample with the $i$-th group B sample for 101 values of $\alpha$. The ablation without separable terms uses the same protocol with the position dependence carried by $\bm\Phi(\cdot,\mat x)$.

\textit{Quantum dynamics.} The official datasets have $L=4$, $N=128$ and $\Delta T=0.1$~ms. KAFlow uses one Strang step per observation interval, a potential $v(x)$ with 12 cubic B-spline intervals plus linear and constant terms, the kinetic dictionary $\{\frac12,1,\dots,4\}$ with softplus coefficients and, for NLSE, the Kerr polynomial $\kappa_1\rho+\kappa_2\rho^2$. The losses are the relative $\ell_2$ loss on wave functions (full) and the Kullback--Leibler divergences of position and momentum probability mass functions (pos~\&~mom, pos), computed in double precision; 5,000 Adam iterations are used. The closed forms in Table~\ref{tab:hamiltonian} are obtained by polynomial fits of the learned potential (DW) and by the distillation of the main text with a profile-likelihood refinement of the Kerr coefficient (NLSE), for which the stored single-precision probability mass functions are renormalized in double precision.

\textit{Variable-order fractional operator.} Inputs are drawn from the official KANO generator with the same families and normalization (2,000 training and 400 held-out pairs from group A, 400 test pairs from group B). Outputs are computed in double precision with the Kohn--Nirenberg quantization on the full frequency grid; applying the same routine to the $\ten G_2$ symbol reproduces the official $\ten G_2$ labels to $2.8\times10^{-7}$. KANO follows the official two-stage recipe: a multilayer-perceptron symbol (hidden width 32, 2 layers) trained sample by sample with Adam (400 epochs, $10^{-3}$, plateau scheduler), distillation into a KAN symbol of width $[[2,0],[1,1]]$, and fine-tuning (100 epochs at $10^{-1}$ and 10 epochs at $10^{-5}$); KANO with an MLP symbol is the model after the first stage. FNO follows the official recipe (batch of 16, Adam $10^{-3}$, 1,000 epochs, plateau scheduler).

\textit{Nonlocal and dispersive dynamics.} Trajectories are generated with a pseudo-spectral Lawson fourth-order Runge--Kutta solver (20 sub-steps per observation, $2/3$ dealiasing) on 128 points with $L=2\pi$ and $\Delta t=0.02$ (Burgers, fractional Burgers) or 0.01 (Benjamin--Ono, Korteweg--de Vries, Whitham). Initial conditions are random Fourier series with wavenumbers 1 to 4, zero mean and unit standard deviation. Of the 256 training trajectories, the last 32 are used for validation. For these translation-invariant equations KAFlow uses $R=0$, 2 nonlocal feature channels and a learned flux multiplier $d(\xi)\in\mathrm{span}\{1,i\xi\}$, that is, $\ten N(\hat{\mat u})=d\cdot\ten F_m[\bm\Phi^{\text{nl}}(\mat u,\Op(\mu)\mat u)]$, with 4 Lawson sub-steps per observation interval. It trains 3 random initializations for 600 one-step iterations, selects the one with the lowest validation rollout loss over the full horizon, and continues with rollout lengths 4 and 16 (600 iterations each, batch of 32). FNO (width 32 or 64, 16 or 32 modes, learning rate $10^{-3}$ or $3\times10^{-3}$) is selected on the same validation trajectories and trained with the same rollout curriculum for 4,000 iterations. PDE-FIND and Weak SINDy use PySINDy 1.7.5; the strong-form library and $u_t$ are projected onto the resolved band of the data, and the selected equations are integrated with the same pseudo-spectral solver.

\begin{table}[h]
\centering
\caption{\textbf{Main symbols and notation.}}
\small
\begin{tabular}{p{0.32\textwidth} p{0.6\textwidth}}
\toprule
Symbol & Meaning\\
\midrule
$D=(-L/2,L/2)^d$ & Periodic spatial domain\\
$\mathcal Y$, $\Xi$, $m$ & Uniform grid, retained modes, resolution\\
$\ten F_m$, $\ten F_m^{-1}$ & Discrete Fourier transform and its inverse\\
$\boldsymbol p(\mat x,\bm\xi)$ & Pseudo-differential symbol\\
$\Op(\boldsymbol p)$ & Kohn--Nirenberg quantization (Eq.~\ref{eq:kn})\\
$R$ & Number of separable position-dependent terms\\
$\ell(\bm\xi)$ & Translation-invariant spectral symbol\\
$a_r(\mat x)$, $m_r(\bm\xi)$ & Position and frequency factors of the $r$-th separable term\\
$\mu(\bm\xi)$, $J$ & Symbols and number of the nonlocal feature channels\\
$\bm\Phi$, $\bm\Phi^{\text{nl}}$ & Separable KAN nonlinearity and its gauge-fixed part (Eq.~\ref{eq:gauge})\\
$\mathcal P$, $c_p$, $r_i$ & Exponent dictionary, dictionary and B-spline residual coefficients\\
$E(\xi)$, $\xi_{\text{ref}}$, $s$ & Training energy spectrum, energy-weighted root-mean-square frequency, symbol scale\\
$T(\xi)$, $v(x)$, $\kappa_{1,2}$ & Kinetic symbol, potential and Kerr coefficients of the quantum model\\
$\Delta t$ & Observation time step\\
\bottomrule
\end{tabular}
\end{table}

\end{document}